\documentclass[twoside, 11pt]{article}

\usepackage[
  colors=true,
  theorems=true,
  code=true,
  draft=false,
  bibstyle=plainnat,
  style=plain,
  numberwithin=none,
  numbering=independent,
]{document}

\usepackage{indentfirst}

\newcommand{\eqname}[1]{\textnormal{(#1)}\qquad}

\newcommand{\defined}{\vcentcolon=}

\newcommand{\reals}{\mathbb{R}}

\newcommand{\Mat}[1][]{\ifthenelse{\equal{#1}{}}{\textnormal{Mat}}{\textnormal{Mat}(#1)}}

\renewcommand{\so}[1]{\mathfrak{so}(#1)}

\newcommand{\tangent}[1]{
    \ifthenelse{\equal{#1}{}}
    {{T}}
    {{T_{#1}}}
}
\newcommand{\dualtangent}[1]{
    \ifthenelse{\equal{#1}{}}
    {{T^*}}
    {{T_{#1}^*}}
}

\newcommand{\Adjoint}[1]{
    \ifthenelse{\equal{#1}{}}
    {\textnormal{Ad}}
    {\textnormal{Ad}_{#1}}
}

\newcommand{\adjoint}[1]{
    \ifthenelse{\equal{#1}{}}
    {\textnormal{ad}}
    {\textnormal{ad}_{#1}}
}
\newcommand{\coAdjoint}[1]{
    \ifthenelse{\equal{#1}{}}
    {\textnormal{Ad}^*}
    {\textnormal{Ad}^*_{#1}}
}

\newcommand{\coadjoint}[1]{
    \ifthenelse{\equal{#1}{}}
    {\textnormal{ad}^*}
    {\textnormal{ad}^*_{#1}}
}

\newcommand{\dif}{\mathrm{d}}

\newcommand{\Set}[1]{\left\{#1\right\}}
\newcommand{\given}{\,\middle|\,}

\newcommand{\abs}[1]{\left|#1\right|}
\newcommand{\norm}[1]{\left\lVert#1\right\rVert}

\newcommand{\inverse}{{-1}}

\newcommand{\transpose}{\intercal}

\makeatletter
\DeclareFontFamily{OMX}{MnSymbolE}{}
\DeclareSymbolFont{MnLargeSymbols}{OMX}{MnSymbolE}{m}{n}
\SetSymbolFont{MnLargeSymbols}{bold}{OMX}{MnSymbolE}{b}{n}
\DeclareFontShape{OMX}{MnSymbolE}{m}{n}{
    <-6>  MnSymbolE5
   <6-7>  MnSymbolE6
   <7-8>  MnSymbolE7
   <8-9>  MnSymbolE8
   <9-10> MnSymbolE9
  <10-12> MnSymbolE10
  <12->   MnSymbolE12
}{}
\DeclareFontShape{OMX}{MnSymbolE}{b}{n}{
    <-6>  MnSymbolE-Bold5
   <6-7>  MnSymbolE-Bold6
   <7-8>  MnSymbolE-Bold7
   <8-9>  MnSymbolE-Bold8
   <9-10> MnSymbolE-Bold9
  <10-12> MnSymbolE-Bold10
  <12->   MnSymbolE-Bold12
}{}

\let\llangle\@undefined
\let\rrangle\@undefined
\DeclareMathDelimiter{\llangle}{\mathopen}%
                     {MnLargeSymbols}{'164}{MnLargeSymbols}{'164}
\DeclareMathDelimiter{\rrangle}{\mathclose}%
                     {MnLargeSymbols}{'171}{MnLargeSymbols}{'171}
\makeatother

\newcommand{\identity}[1]{
    \ifthenelse{\equal{#1}{}}
    {\mathrm{I}}
    {\mathrm{I}_{#1}}
}

\newcommand{\ProbabilitySpace}[1]{\mathcal{P}(#1)}

\newcommand{\indicator}[1]{\mathbf{1}_{\left\{#1\right\}}}
\NewDocumentCommand{\Probability}{o m}{
  \mathsf{P}
  \IfValueT{#1}{_{#1}}
  \IfNoValueT{#1}{}
  \!\left(#2\right)
}
\NewDocumentCommand{\Expectation}{o m}{
  \mathsf{E}
  \IfValueT{#1}{_{#1}}
  \IfNoValueT{#1}{}
  \!\left(#2\right)
}
\NewDocumentCommand{\Variance}{o m}{
  \mathsf{Var}
  \IfValueT{#1}{_{#1}}
  \IfNoValueT{#1}{}
  \!\left(#2\right)
}
\NewDocumentCommand{\Covariance}{o m}{
  \mathsf{Cov}
  \IfValueT{#1}{_{#1}}
  \IfNoValueT{#1}{}
  \!\left(#2\right)
}
\DeclareMathOperator*{\logsumexp}{logsumexp}
\newcommand{\KLDivergence}[2]{\textnormal{KL}\left(#1 \| #2\right)}

\newcommand{\subjectto}{\textnormal{s.t.}}

\newcommand{\inline}[1]{%
  \begingroup
  \let\left\relax
  \let\right\relax
  #1%
  \endgroup
}

\newcommand{\raisedtri}[1]{%
  \raisebox{0.2ex}{\scalebox{0.6}{$#1\vartriangle$}}%
}

\makeatletter
\@tfor\letter:=abcdefghijklmnopqrstuvwxyz\do{%
  \expandafter\edef\csname tt\letter\endcsname{\noexpand\textnormal{\letter}}%
  \expandafter\edef\csname sf\letter\endcsname{\noexpand\mathsf{\letter}}%
  \expandafter\edef\csname rm\letter\endcsname{\noexpand\mathrm{\letter}}%
  \expandafter\edef\csname delta\letter\endcsname{%
    \noexpand\mathchoice
      {\noexpand\raisedtri\displaystyle\noexpand\mkern-1mu \letter}%
      {\noexpand\raisedtri\textstyle\noexpand\mkern-1mu \letter}%
      {\noexpand\raisedtri\scriptstyle\noexpand\mkern-1mu \letter}%
      {\noexpand\raisedtri\scriptscriptstyle\noexpand\mkern-1mu \letter}%
  }%
}
\@tfor\letter:=ABCDEFGHIJKLMNOPQRSTUVWXYZ\do{%
  \expandafter\edef\csname tt\letter\endcsname{\noexpand\textnormal{\letter}}%
  \expandafter\edef\csname cl\letter\endcsname{\noexpand\mathcal{\letter}}%
  \expandafter\edef\csname sf\letter\endcsname{\noexpand\mathsf{\letter}}%
  \expandafter\edef\csname bb\letter\endcsname{\noexpand\mathbb{\letter}}%
  \expandafter\edef\csname rm\letter\endcsname{\noexpand\mathrm{\letter}}%
  \expandafter\edef\csname delta\letter\endcsname{%
    \noexpand\mathchoice
      {\noexpand\raisedtri\displaystyle\noexpand\mkern-1mu \letter}%
      {\noexpand\raisedtri\textstyle\noexpand\mkern-1mu \letter}%
      {\noexpand\raisedtri\scriptstyle\noexpand\mkern-1mu \letter}%
      {\noexpand\raisedtri\scriptscriptstyle\noexpand\mkern-1mu \letter}%
  }%
}
\makeatother

\newcommand{\TPR}{\textnormal{TPR}}
\newcommand{\FPR}{\textnormal{FPR}}

\newcommand{\LCB}{\textnormal{LCB}}
\newcommand{\UCB}{\textnormal{UCB}}
\newcommand{\CE}{\textnormal{CE}}

\makeatletter

\newlength{\figureframearc}
\newtcolorbox{figureframe}[2][]{%
  boxrule=1pt,            
  colframe=black,         
  colback=white,          
  colbacktitle=gray!30,   
  coltitle=black,         
  fonttitle=\bfseries\small,
  titlerule=0pt,          
  arc=\figureframearc,    
  boxsep=0pt,             
  left=0pt,               
  right=0pt,              
  top=2pt,                
  bottom=0pt,             
  halign=center,          
  valign=bottom,          
  toptitle=1mm,
  bottomtitle=1mm,
  title={\parbox[c][1.2em][c]{\linewidth}{\centering
    \smash{\raisebox{-0.5\fontcharht\font`H}{#2}}}},
  #1%
}

\newsavebox{\figurepanel@box}
\NewDocumentEnvironment{figurepanel}{O{} m}{%
  \sbox{\figurepanel@box}{\includegraphics[width=\linewidth, #1]{#2}}%
  \begin{tikzpicture}[x=\wd\figurepanel@box, y=\ht\figurepanel@box]
    \begin{scope}
      \clip (0,1)
        -- ([yshift=\figureframearc]0,0)
        arc[start angle=180, end angle=270, radius=\figureframearc]
        -- ([xshift=-\figureframearc]1,0)
        arc[start angle=270, end angle=360, radius=\figureframearc]
        -- (1,1) -- cycle;
      \node[anchor=south west, inner sep=0pt] (panel) at (0,0) {\usebox{\figurepanel@box}};
    \end{scope}
}{%
  \end{tikzpicture}%
}

\newsavebox{\speechbubble@box}
\pgfkeys{
  /speechbubble/.is family, /speechbubble,
  at/.store in         = \speechbubble@at,
  tip/.store in        = \speechbubble@tip,
  width/.store in      = \speechbubble@width,
  tail width/.store in = \speechbubble@tailwidth,
  tail shift/.store in = \speechbubble@tailshift,
  tail bend/.store in  = \speechbubble@tailbend,
  align/.is choice,
  align/left/.code     = {\def\speechbubble@align{\raggedright}},
  align/center/.code   = {\def\speechbubble@align{\centering}},
  align/right/.code    = {\def\speechbubble@align{\raggedleft}},
  align/justify/.code  = {\def\speechbubble@align{}},
  font/.store in       = \speechbubble@font,
  inner sep/.store in  = \speechbubble@innersep,
  radius/.store in     = \speechbubble@radius,
  line width/.store in = \speechbubble@linewidth,
  draw/.store in       = \speechbubble@draw,
  fill/.store in       = \speechbubble@fill,
  text/.store in       = \speechbubble@text,
  defaults/.style = {
    at={(0,1)}, tip={(0,0)}, width=2.5cm,
    tail width=0.3cm, tail shift=0, tail bend=0.6,
    align=center, font=\small, inner sep=3pt, radius=4pt,
    line width=0.8pt, draw=black, fill=white, text=black,
  },
}
\newcommand{\speechbubble}[2][]{%
  \begingroup
  \pgfkeys{/speechbubble/.cd, defaults}%
  \pgfkeys{/speechbubble/.cd,#1}
  \path \speechbubble@at;  \pgfgetlastxy{\speechbubble@ax}{\speechbubble@ay}%
  \path \speechbubble@tip; \pgfgetlastxy{\speechbubble@tx}{\speechbubble@ty}%
  \pgfmathsetlengthmacro{\speechbubble@w}{\speechbubble@width}%
  \pgfmathsetlengthmacro{\speechbubble@pad}{\speechbubble@innersep}%
  \pgfinterruptpicture
    \global\sbox{\speechbubble@box}{%
      \parbox[b]{\dimexpr\speechbubble@w-2\dimexpr\speechbubble@pad\relax\relax}{%
        \speechbubble@font\speechbubble@align\color{\speechbubble@text}%
        \strut#2\strut}}%
  \endpgfinterruptpicture
  \pgfmathsetlengthmacro{\speechbubble@h}%
    {\ht\speechbubble@box + \dp\speechbubble@box + 2*\speechbubble@pad}%
  \pgfmathsetlengthmacro{\speechbubble@xl}{\speechbubble@ax - \speechbubble@w/2}%
  \pgfmathsetlengthmacro{\speechbubble@xr}{\speechbubble@ax + \speechbubble@w/2}%
  \pgfmathsetlengthmacro{\speechbubble@yt}{\speechbubble@ay + \speechbubble@h}%
  \pgfmathsetlengthmacro{\speechbubble@bc}%
    {\speechbubble@ax + \speechbubble@tailshift*\speechbubble@w/2}%
  \pgfmathsetlengthmacro{\speechbubble@bl}{\speechbubble@bc - (\speechbubble@tailwidth)/2}%
  \pgfmathsetlengthmacro{\speechbubble@br}{\speechbubble@bc + (\speechbubble@tailwidth)/2}%
  \pgfmathsetlengthmacro{\speechbubble@r}{\speechbubble@radius}%
  \pgfmathsetlengthmacro{\speechbubble@q}{(1-0.5523)*\speechbubble@r}%
  \pgfmathsetlengthmacro{\speechbubble@ayr}{\speechbubble@ay - \speechbubble@r}%
  \pgfmathsetlengthmacro{\speechbubble@ayq}{\speechbubble@ay - \speechbubble@q}%
  \pgfmathsetlengthmacro{\speechbubble@blr}{\speechbubble@bl - \speechbubble@r}%
  \pgfmathsetlengthmacro{\speechbubble@blq}{\speechbubble@bl - \speechbubble@q}%
  \pgfmathsetlengthmacro{\speechbubble@brr}{\speechbubble@br + \speechbubble@r}%
  \pgfmathsetlengthmacro{\speechbubble@brq}{\speechbubble@br + \speechbubble@q}%
  \pgfmathsetlengthmacro{\speechbubble@cy}%
    {\speechbubble@ay + \speechbubble@tailbend*(\speechbubble@ty - \speechbubble@ay)}%
  \pgfmathsetlengthmacro{\speechbubble@clx}%
    {\speechbubble@tx + \speechbubble@tailbend*(\speechbubble@bl - \speechbubble@tx)}%
  \pgfmathsetlengthmacro{\speechbubble@crx}%
    {\speechbubble@tx + \speechbubble@tailbend*(\speechbubble@br - \speechbubble@tx)}%
  \draw[line width=\speechbubble@linewidth, draw=\speechbubble@draw,
        fill=\speechbubble@fill, line join=round]
    (\speechbubble@brr, \speechbubble@ay)
    -- ([xshift=-\speechbubble@r]\speechbubble@xr, \speechbubble@ay)
       arc[start angle=270, end angle=360, radius=\speechbubble@r]
    -- ([yshift=-\speechbubble@r]\speechbubble@xr, \speechbubble@yt)
       arc[start angle=0, end angle=90, radius=\speechbubble@r]
    -- ([xshift=\speechbubble@r]\speechbubble@xl, \speechbubble@yt)
       arc[start angle=90, end angle=180, radius=\speechbubble@r]
    -- ([yshift=\speechbubble@r]\speechbubble@xl, \speechbubble@ay)
       arc[start angle=180, end angle=270, radius=\speechbubble@r]
    -- (\speechbubble@blr, \speechbubble@ay)
    .. controls (\speechbubble@blq, \speechbubble@ay)
            and (\speechbubble@bl, \speechbubble@ayq) .. (\speechbubble@bl, \speechbubble@ayr)
    .. controls (\speechbubble@bl, \speechbubble@cy)
            and (\speechbubble@clx, \speechbubble@ty) .. (\speechbubble@tx, \speechbubble@ty)
    .. controls (\speechbubble@crx, \speechbubble@ty)
            and (\speechbubble@br, \speechbubble@cy) .. (\speechbubble@br, \speechbubble@ayr)
    .. controls (\speechbubble@br, \speechbubble@ayq)
            and (\speechbubble@brq, \speechbubble@ay) .. (\speechbubble@brr, \speechbubble@ay)
    -- cycle;
  \node[anchor=south, inner sep=\speechbubble@pad]
    at (\speechbubble@ax, \speechbubble@ay) {\usebox{\speechbubble@box}};
  \endgroup
}

\makeatother

\graphicspath{{figure/}}

\purpose{Research White Paper}
\title{Accuracy and Uncertainty\\in Machine Deception Detection}
\author{Postdoctoral Researcher\\Dr. Heng-Sheng Chang\hspace{1em}hschang@illinois.edu}
\institute{Coordinated Science Laboratory\\University of Illinois Urbana-Champaign}
\date{\today}

\begin{document}

  \pagenumbering{gobble}
  \newgeometry{
    textheight=9in,
    textwidth=5.5in,
    hcentering,
    top=1in,
    headheight=12pt,
    headsep=25pt,
    footskip=30pt
  }
  \makeatletter
  \begin{titlepage}
    \raggedright
    \vspace*{0.36\textheight}

    {\LARGE\textsc{\@purpose}\par}

    \vspace{0.08\textheight}

    {\Huge\bfseries\@title\par}

    \vspace{0.08\textheight}

    \Large
    {\@author\par}

    \vspace{0.5\baselineskip}

    {\@institute\par}

    \vspace{\baselineskip}

    {\@date\par}
  \end{titlepage}
  \makeatother
  \restoregeometry

  \pagenumbering{roman}
  \pagestyle{plain}

  \begin{abstract}
    A deception detector is only as useful as the trust its alarms deserve.
    Like a polygraph, detectors that read a language model's internal activations can catch deceptive responses that look honest.
    Yet current detectors depend on which layers they read, weaken on broader tests, and do not state how certain their alarms are.
    This white paper proposes the Machine Deception Detector, which targets both detection accuracy and the uncertainty of detector scores.
    The detector reads every layer of the language model, learns which tokens carry evidence of deception using only response labels, and calibrates its alarm threshold so that its false--positive rate stays below a prescribed level.
    An ensemble of such detectors measures the uncertainty of each score and raises an alarm only when its members agree, with a confidence level that sets how certain each alarm must be.
    The proposed evaluation trains on instructed honest and deceptive statements, tests on five unseen deceptive tasks of Liars' Bench read by six Qwen3 models from 0.6B to 32B parameters, and is designed to also show which layers carry the deception feature and how this changes with model size.
  \end{abstract}

  \setcounter{tocdepth}{3}
  \tableofcontents
  \newpage

  \pagenumbering{arabic}
  \pagestyle{fancy}

  \section{Introduction}
  \labelsec{introduction}

  Language models can be deceptive, stating what they represent as false: they conceal misconduct, misreport what they know, and claim traits they lack \parencite{kretschmar2025liars,goldowsky2025detecting}.
  A lie written to sound honest is hard to catch from outputs alone.
  Like a polygraph reading physiological signals (\reffig{polygraph}), a detector can read the model's internal activations to flag deception in deployment, audits, and safety cases.
  Such a detector must meet two requirements: it must be accurate at a false--positive rate low enough for routine traffic, and state how far each alarm can be trusted.

  Existing detectors meet the first requirement only in narrow settings.
  A linear probe on a single middle layer of Llama-3.3-70B detects strategic deception with an \emph{area under the receiver operating characteristic curve} (AUROC) above 0.96 and catches most deceptive responses at a 1\% false--positive rate \parencite{goldowsky2025detecting}.
  However, the authors note that this result is unstable: small methodological changes, such as probing two layers later, drop it to nearly zero on two tasks.
  These results weaken on broader tests.
  On Liars' Bench, spanning four open-weight models and seven settings, no method, white- or black-box, is above chance in every setting, and a single-layer probe ranks last \parencite{kretschmar2025liars}.

  Three weaknesses recur.
  First, single-layer probes are brittle: the best layer changes across models and tasks, whereas combining layers improves AUROC by up to 78\% where single layers fail \parencite{nordby2026linear}.
  Second, probe directions mix deception with its scenario, so probes also fire on honest discussion of misconduct or sensitive topics \parencite{goldowsky2025detecting}.
  Third, each detector reports one score from one fitted model, although detectors that differ only in training randomness can disagree on the same response \parencite{marx2020predictive}, so none states whether its alarm would survive retraining.

  This white paper addresses both requirements with white-box detectors on the residual stream of a language model, organized in two thrusts:
  \begin{enumerate}[label=(\arabic*)]
    \item \textbf{Detection Accuracy:} The first thrust aims to increase detection accuracy.
    It scores each token with a logistic probe per layer, pools the scores with an aggregator inside the training loss, and compares aggregators and regularization norms to find the best combination.
    The detector learns from response labels alone, reveals which layers carry the deception feature, and keeps its false--positive rate below a prescribed level on transcripts like those it is calibrated on.
    \item \textbf{Uncertainty in Detector Scores:} The second thrust aims to quantify how far each alarm can be trusted.
    An ensemble of first-thrust detectors estimates the mean and spread of each score over training randomness, and its \emph{lower confidence bound} (LCB) raises an alarm only when the members agree.
    Through its confidence level $n$, the LCB provides each alarm with a \emph{certainty}, the chance that a retrained detector raises the same alarm, of at least $n^{2}/(1 + n^{2})$.
  \end{enumerate}

  The rest of the white paper is organized as follows: 
  \refsec{preliminaries} casts deception detection as a polygraph examination of a language model and defines its mathematical objects; 
  \refsec{approach} presents the two thrusts; 
  \refsec{evaluation} specifies the evaluation tasks and metrics;
  \refsec{preliminary-evidence} reports preliminary evidence; 
  \refsec{milestones} outlines the work plan and required resources; 
  \refsec{limitations} discusses limitations and extensions; 
  and \refsec{open-questions} outlines applications and open research questions.

  \begin{figure}[b]
    \vspace*{-\baselineskip}

    \centering
    \includegraphics{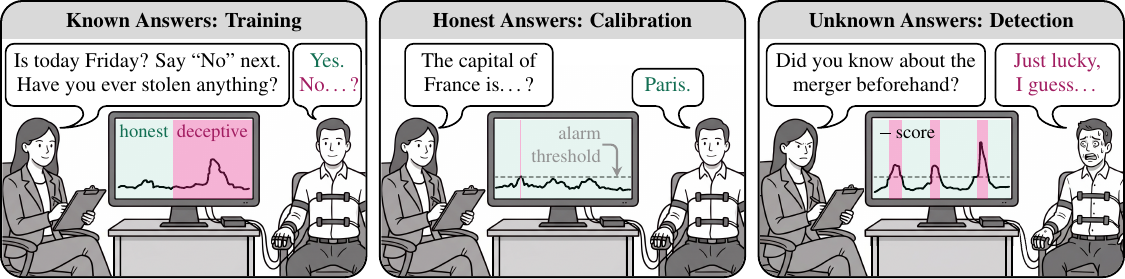}%
    \caption{
      A polygraph examination and its machine counterpart, which reads a language model's embeddings instead of physiological signals. 
      (Left) The examiner learns how signals differ between truthful answers and instructed lies, as a detector is trained on labeled transcripts.
      (Middle) Honest answers set the cutoff, or alarm threshold. 
      (Right) An answer whose signal exceeds the cutoff is deemed deceptive, as is a transcript scoring above it.%
    }
    \labelfig{polygraph}
  \end{figure}

  \section{Preliminaries --- A Polygraph for Language Models}
  \labelsec{preliminaries}

  The setup follows the protocol of a polygraph examination, in which an examiner judges whether the answers of a human subject are lies \parencite{national2003polygraph}.
  To detect a lie, the examiner looks past what the subject says and records physiological signals, such as heart rate and skin conductance, that arise from the subject's internal state while answering.
  A subject can make a lie sound honest, but the internal state that produces the words also drives these physiological signals, which can therefore reveal the deception.
  The examination begins with a baseline (\reffig{polygraph}, left): the subject answers questions whose truth is known, some truthfully and some with an instructed lie, so that the examiner learns how this subject's signals differ between honest and deceptive answers.
  The examiner then sets a cutoff from further truthful answers (\reffig{polygraph}, middle), so that this subject's honest answers rarely exceed it.
  Only then are the relevant questions asked (\reffig{polygraph}, right), and an answer is called deceptive when its signal, read against the baseline, exceeds the cutoff.

  In the machine counterpart, the subject is a language model, and the embeddings that arise as it responds to a prompt take the place of the physiological signals.
  A deceptive answer can sound honest in the words that the final layers produce, while its intent remains encoded in the embeddings of the early and middle layers, where it is much harder to hide \parencite{halawi2024overthinking,marks2023geometry}.
  The baseline corresponds to a labeled dataset of honest and deceptive transcripts, on which a deception detector is trained.
  The cutoff corresponds to an alarm threshold, set on a separate set of ordinary honest transcripts so that the model's honest answers rarely trigger it.
  The detector then scores new transcripts whose labels are unknown, and flags those whose score exceeds the threshold.
  The rest of this section defines the mathematical objects that formalize the meaning of the labeled datasets, the machine subject, and the deception detector.

  \paragraph{Labeled Datasets}
  A transcript is a sequence of tokens $Z_{1:T} \defined \Set{Z_{t} \in \bbO : t = 1, \dots, T}$ over a vocabulary $\bbO$.
  It consists of a context instruction $Z^{\ttc}$, a prompt question $Z^{\ttp}$, and a response answer $Z^{\ttr}$ of length $T^{\ttr}$; we write $Z \defined \Set{Z^{\ttc}, Z^{\ttp}, Z^{\ttr}}$ and omit sequence indices when no confusion arises.
  Each transcript carries a label $Y \in \Set{0, 1}$: $Y = 1$ if the response $Z^\ttr$ is annotated deceptive, meaning that it conceals or misreports information because the context $Z^\ttc$ instructs or pressures it to, and $Y = 0$ if it is annotated honest, meaning that it reports the information faithfully.
  The training dataset $\clD \defined \inline{\Set{(Z^{(i)}, Y^{(i)})}}_{i=1}^{I}$ holds $I$ transcript--label pairs; a validation dataset $\clD_\ttv$ is held out from the same source, the honest-only datasets $\clD_\tth$ and $\clD_\tta$ calibrate and audit the threshold, and the test dataset $\clD_\ttt$ comes from different tasks.

  \paragraph{Machine Subject}
  Let $f$ be a subject language model with $L$ layers and embedding dimension $\sfd$.
  Reading the transcript up to response token $t$, that is, the sequence $Z_{1:T_t} \defined \Set{Z^{\ttc}, Z^{\ttp}, Z^{\ttr}_{1:t}}$ of length $T_t = \abs{Z^{\ttc}} + \abs{Z^{\ttp}} + t$, the model produces a residual-stream embedding $X^{(\ell)}_{t}\in\reals^\sfd$ at each layer $\ell\in\bbL\defined\Set{1,\dots,L}$,
  \begin{equation*}
    f^{(\ell)} : \bbO^{T_t} \to \reals^{\sfd},
    \qquad
    Z_{1:T_t} \mapsto X^{(\ell)}_{t} = f^{(\ell)}(Z_{1:T_t}),
    \qquad t = 1, \dots, T^{\ttr}.
  \end{equation*}
  The concatenated embedding $X\defined \inline{\Set{X^{(\ell)}_{t} : \ell\in\bbL; t = 1, \dots, T^{\ttr}}}\in\reals^{T^{\ttr}\times L \times \sfd}$ is the representation of the response $Z^{\ttr}$ produced by the model $f$, denoted by $X = f(Z)$ with a slight abuse of notation.
  Embedding $X_t^{(1)}$ denotes the token embedding after the first attention and feedforward operation, and $X_t^{(L)}$ is the final embedding, from which a linear projection and a softmax give the distribution of the next token $Z^{\ttr}_{t+1}$.
  We write this distribution as
  \begin{equation*}
    f^+ : \bbO^{T_t} \to \ProbabilitySpace{\bbO},
    \qquad
    Z_{1:T_t} \mapsto \Probability{Z^{\ttr}_{t+1} \given Z_{1:T_t}} = f^+(Z_{1:T_t}),
  \end{equation*}
  where $\ProbabilitySpace{\bbO}$ is the set of probability distributions over $\bbO$, from which the next token in the response $Z^{\ttr}_{t+1}$ is sampled.
  
  \paragraph{Deception Detector}
  Given a transcript $Z$, a detector $g_{\theta}$ with trainable parameters $\theta$ reads the embeddings $X = f(Z)$ of the machine subject $f$ and returns a deception score $s \in \reals$, and a threshold $s^{*} \in \reals$ flags the response as deceptive, $\hat{Y} = 1$, when the score reaches it:
  \begin{equation*}
    g_\theta:\reals^{T^{\ttr}\times L \times \sfd}\to\reals,
    \qquad
    X\mapsto s = g_{\theta}(X);
    \qquad
    \hat{Y} = \indicator{s \geq s^{*}} \in \Set{0, 1}.
  \end{equation*}
  \begin{remark*}{Off-policy transcripts}
    The responses in the labeled datasets may have been written by models other than $f$.
    Their labels come from the instruction in the context or from the beliefs of the model that wrote them.
    The detector $g_\theta$ thus learns how $f$ represents the responses it reads, like a polygraph examiner recording a subject who reads aloud answers labeled as lies or truths.
    Once trained, it is intended to score $f$'s own responses to any prompt, even without a context ($Z^{\ttc} = \emptyset$).
  \end{remark*}

  \section{Technical Approach --- Machine Deception Detector}
  \labelsec{approach}

  A deception detector should flag as many deceptive responses as possible while rarely flagging honest ones; this is the aim of the first thrust (\refsec{thrust-accuracy}).
  A score that a detector assigns to a response alone, however, does not show how far it can be trusted, since detectors that differ only in initialization or training sample may assign the same response very different scores \parencite{marx2020predictive,d2022underspecification,zou2023representation}.
  Quantifying this uncertainty and using it to make detection more reliable is the aim of the second thrust (\refsec{thrust-uncertainty}).
  Each thrust is stated as an optimization problem with a hypothesis and a goal, followed by the architecture of the detector, its training objective, and its inference.

  \subsection{Thrust 1 --- Detection Accuracy}
  \labelsec{thrust-accuracy}

  The first thrust aims to maximize the \emph{recall} of a detector $g_\theta$ on a machine subject $f$ while keeping its \emph{false--positive rate} (FPR) on honest responses below a prescribed level.
  On a dataset $\clD$ and at a threshold $s^{*}$, the recall, or \emph{true--positive rate} (TPR), and the FPR are the probabilities that the detector flags a deceptive and an honest response, respectively:
  \begin{equation*}
    \TPR(\clD; f, g_\theta, s^{*}) \defined \Probability[\clD]{g_{\theta}(f(Z)) \geq s^{*} \given Y = 1},
    \qquad
    \FPR(\clD; f, g_\theta, s^{*}) \defined \Probability[\clD]{g_{\theta}(f(Z)) \geq s^{*} \given Y = 0}.
  \end{equation*}
  For a false--positive level $\gamma \in (0, 1)$, the alarm threshold $s^{*}_{\gamma}$ is the smallest threshold whose FPR on $\clD$ does not exceed $\gamma$,
  \begin{equation*}
    s^{*}_{\gamma} \defined \inf \Set{s^{*} \in \reals : \FPR(\clD; f, g_\theta, s^{*}) \leq \gamma}.
  \end{equation*}

  \paragraph{Problem Formulation}
  For a given machine subject $f$, an ideal detector $g_\theta$ maximizes the recall on the training dataset $\clD$ subject to the false--positive constraint:
  \begin{equation}
    \begin{aligned}
      \max_{\theta} \quad & \TPR(\clD; f, g_\theta, s^{*}_{\gamma}) \\
      \subjectto \quad & \FPR(\clD; f, g_\theta, s^{*}_{\gamma}) \leq \gamma.
    \end{aligned}
    \labeleq{accuracy-problem}
  \end{equation}

  \begin{hypothesis}{Aggregation, regularization, and deception features}{accuracy}
    \begin{enumerate}[label=(\arabic*)]
      \item The accuracy of a detector depends on two design choices: the aggregator that pools token scores into a response score, and the norm that regularizes the detector weights.
      \item The deception feature is concentrated in a few layers of the machine subject, most likely in its early and middle layers rather than its final ones.
    \end{enumerate}
  \end{hypothesis}

  \begin{goal}{Accurate and interpretable detector}{accuracy}
    \begin{enumerate}[label=(\arabic*)]
      \item Design and train a detector that achieves high recall across false--positive levels $\gamma$.
      \item Locate the layers that carry the deception feature from each layer's contribution to the detector score and its stability across bootstrap refits.
    \end{enumerate}
  \end{goal}

  \subsubsection{Architecture --- Multi-Layer Aggregation Detector}
  \labelsec{accuracy-architecture}

  \paragraph{Layer Probe}
  Logistic probes, affine readouts of the embedding at one layer trained with logistic regression \parencite{alain2016understanding}, detect false statements and strategic deception from the intermediate layers of language models \parencite{azaria2023internal,marks2023geometry,goldowsky2025detecting}.
  However, because the optimal layer varies across responses, relying on a single layer is brittle, whereas combining multiple layers improves detection \parencite{nordby2026linear}.

  \begin{assumption}{Position invariance}{position-invariance}
    The embedding of each response token $t$ already encodes the context needed to score it, so one token-level detector applies to every token, invariantly across positions.
  \end{assumption}

  \refasm{position-invariance} rests on the linear representation hypothesis, under which a concept is encoded as a fixed direction in the residual stream \parencite{park2023linear}; accordingly, probes for honesty, deception, and high-stakes content apply the same weights at every position \parencite{zou2023representation,goldowsky2025detecting,mckenzie2026detecting}.
  A token-level detector $\bar{g}_{\theta}$ therefore scores the embedding $X_{t} \defined \inline{\Set{X^{(\ell)}_{t}}}_{\ell \in \bbL}$ of each response token with one logistic probe per layer and sums their logits:
  \begin{equation*}
    s_{t} = \bar{g}_{\theta}(X_{t}) \defined b + \sum_{\ell \in \bbL} w_{\ell}^{\transpose} X^{(\ell)}_{t},
    \qquad t = 1, \dots, T^{\ttr},
    \qquad \theta \defined \left(b, \Set{w_{\ell}}_{\ell \in \bbL}\right),
  \end{equation*}
  where $w_{\ell} \in \reals^{\sfd}$ is the probe weight at layer $\ell$ and $b \in \reals$ a shared bias, since per-layer biases would add up to one constant.
  The token score $s_{t}\in\reals$ is thus a logit over all layers, high when token $t$ carries evidence of deception.

  \paragraph{Aggregation over Tokens}
  Because the label $Y$ belongs to the whole response while its evidence of deception may sit in only a few tokens, the detector treats the response as a bag of tokens, as in multiple-instance learning \parencite{dietterich1997solving,ilse2018attention}, and an aggregator $\clA$ maps the token scores to the response score as
  \begin{equation*}
    s = g_{\theta}(X) \defined \clA\!\left(\bar{g}_{\theta}(X_{1}), \dots, \bar{g}_{\theta}(X_{T^{\ttr}})\right) = \clA(s_{1}, \dots, s_{T^{\ttr}}).
  \end{equation*}
  How best to pool them remains open: deception probes usually average the token scores \parencite{goldowsky2025detecting}, whereas a comparison of pooling methods for high-stakes interactions finds softmax- and attention-weighted pooling the most effective \parencite{mckenzie2026detecting}.
  Three aggregators spanning this range are compared (\reftab{aggregators}):
  \begin{enumerate}[label=(\arabic*)]
    \item Average \parencite{nordby2026linear} weights every token equally; since the probes are affine, it equals the probe on the mean embedding, a baseline linear in the embeddings.
    \item Density, a length-normalized log-sum-exp pooling from multiple-instance learning \parencite{pinheiro2015image}, lies between average and peak and depends on the fraction of high-scoring tokens rather than their number; it rewards sustained high scores, while discounting an isolated spurious spike.
    \item Peak, the softmax pooling of \textcite{mckenzie2026detecting} and among the best in their comparison, shifts weight toward the highest-scoring tokens, so that a single token scoring well above the rest keeps its influence at any response length, whereas average dilutes it and density discounts it.
  \end{enumerate}

  \begin{table}[t]
    \centering
    \caption{Aggregators of the token logits $\Set{s_{t}:t = 1, \dots, T^{\ttr}}$ over a response of length $T^{\ttr}$, with temperature $\tau \defined \nicefrac{1}{\log T^{\ttr}}$.}
    \labeltab{aggregators}
    \vspace*{-0.5\baselineskip}
    \renewcommand{\arraystretch}{1.25}
    \small
    \begin{tabular}{cccc}
      \toprule
      \multirow{2}{*}{Aggregator $\clA$} & Average $\clA_{\textnormal{avg}}$ & Density $\clA_{\textnormal{den}}$ & Peak $\clA_{\textnormal{peak}}$ \\
      & \parencite{nordby2026linear} & \parencite{pinheiro2015image} & \parencite{mckenzie2026detecting} \\
      \cmidrule(lr){1-1}\cmidrule(lr){2-2}\cmidrule(lr){3-3}\cmidrule(lr){4-4}
      Response score $s = \clA(s_{1:T^{\ttr}})$
      &
      $\displaystyle \frac{1}{T^{\ttr}} \sum_{t=1}^{T^{\ttr}} s_{t}$
      &
      $\displaystyle \logsumexp\left(s_{1:T^{\ttr}}\right) - \tau^\inverse$
      &
      $\displaystyle \sum_{t=1}^{T^{\ttr}} \operatorname{softmax}_{t}\!\left(\frac{s_{1:T^{\ttr}}}{\tau}\right) s_{t}$ \\
      \bottomrule
    \end{tabular}%
    \vspace*{-0.5\baselineskip}
  \end{table}

  \subsubsection{Training Objective --- Loss with Regularization}
  \labelsec{accuracy-training}

  \paragraph{Loss}
  The detector $g_\theta$ is trained on $\clD = \inline{\Set{(Z^{(i)}, Y^{(i)})}}_{i=1}^{I}$ by minimizing the regularized response-level cross-entropy
  \begin{equation*}
    \min_{\theta} \quad
      \sfJ(g_\theta;f,\clD) \defined \Expectation[\clD]{\sfJ_{\CE}\!\left(Y, \psi\!\left(g_{\theta}(f(Z))\right)\right)}
      + \frac{\lambda}{ L } \, \sfJ_{\textnormal{reg}}\!\left(\Set{w_{\ell}}_{\ell \in \bbL}\right),
  \end{equation*}
  where the regularizer $\sfJ_{\textnormal{reg}}$ penalizes the probe weights with strength $\lambda \geq 0$, normalized by the number of layers, and
  \begin{align*}
  \eqname{binary cross-entropy}&\sfJ_{\CE}: \Set{0,1}\times(0,1)\to\reals_+,\quad  (y, p) \mapsto \sfJ_{\CE}(y, p)\defined - y \log p - (1 - y) \log (1 - p),\\
  \eqname{logistic function}&\psi: \reals \to (0,1),\quad s \mapsto \psi(s) \defined \frac{1}{1 + e^{-s}}.
  \end{align*}
  The cross-entropy is a smooth surrogate for the classification error: it raises the scores of deceptive responses and lowers those of honest ones, which increases the recall at any fixed false--positive level, while the constraint of \eqref{eq:accuracy-problem} is enforced afterwards by calibrating the threshold (\refsec{accuracy-inference}).
  Placing the aggregator inside the loss lets the label supervise only the response score, so the honest tokens of a deceptive response are never mislabeled.
  With the average aggregator the problem is convex in $\theta$, as the score is affine in the probe weights and both the loss and the regularizers are convex; with density or peak it is not convex in general, so each configuration is fit from several initializations.

  \paragraph{Regularization}
  The regularizer $\sfJ_{\textnormal{reg}}$ sets where the detector locates the deception feature, and three norms are compared (\reftab{norms}): lasso selects embedding features, group lasso whole layers, and sparse-group lasso a few layers and a few features within each, balanced by $\kappa$.
  Under the group norms, layers drop out as $\lambda$ grows, the most important last, so the \emph{group-lasso path} over a grid of $\lambda$, refit on bootstrap resamples of the training pairs, ranks the layers, and a layer carries the deception feature when most refits select it.
  Since embedding norms vary across layers, each layer is weighted by its embedding scale $c_{\ell}$, so that both the penalty and the layer importance $c_{\ell} \norm{w_{\ell}}_{2}$ reflect the layer's contribution to the logit rather than the size of its weights; the unweighted fit ($c_{\ell} \equiv 1$) serves as an ablation.

  \begin{table}[t]
    \centering
    \caption{Regularization norms on the probe weights $\Set{w_{\ell}}_{\ell \in \bbL}$, with mixing weight $\kappa \in (0, 1)$ and embedding scales $c_{\ell}^{2} \defined \inline{\Expectation[\clD]{\norm{X^{(\ell)}_{t}}_{2}^{2}}}$.}
    \labeltab{norms}
    \vspace*{-0.5\baselineskip}
    \renewcommand{\arraystretch}{1.25}
    \small
    \begin{tabular}{cccc}
      \toprule
      \multirow{2}{*}{Regularization norm $\sfJ_{\textnormal{reg}}$} & Lasso & Group lasso & Sparse-group lasso \\
      & \parencite{tibshirani1996regression} & \parencite{yuan2006model} & \parencite{simon2013sparse} \\
      \cmidrule(lr){1-1}\cmidrule(lr){2-2}\cmidrule(lr){3-3}\cmidrule(lr){4-4}
      Value $\sfJ_{\textnormal{reg}}\!\left(\Set{w_{\ell}}_{\ell \in \bbL}\right)$
      &
      $\displaystyle \sum_{\ell \in \bbL} c_{\ell} \norm{w_{\ell}}_{1}$
      &
      $\displaystyle \sqrt{\sfd} \sum_{\ell \in \bbL} c_{\ell} \norm{w_{\ell}}_{2}$
      &
      $\displaystyle (1 - \kappa)  \sum_{\ell \in \bbL} c_{\ell} \norm{w_{\ell}}_{1} + \kappa \sqrt{\sfd}  \sum_{\ell \in \bbL} c_{\ell} \norm{w_{\ell}}_{2} $ \\
      \bottomrule
    \end{tabular}%
    \vspace*{-0.5\baselineskip}
  \end{table}

  \subsubsection{Inference --- Threshold Estimation and Accuracy}
  \labelsec{accuracy-inference}

  \paragraph{Threshold Estimation}
  The unknown threshold $s^{*}_{\gamma}$ is estimated by an order statistic of the scores on the honest-only dataset $\clD_{\tth}$, as in Neyman--Pearson classification \parencite{tong2018neyman}.
  Split conformal prediction guarantees the false--positive rate of this estimate $\hat{s}^{*}_{\gamma}$ in finite samples \parencite{vovk2005algorithmic,angelopoulos2023conformal}, as stated in \refprop{conformal-fpr}.

  \begin{proposition}{Finite-sample false--positive rate}{conformal-fpr}
    Let the detector $g_\theta$ be fixed before calibration, and let $s^{(1)} \leq \dots \leq s^{(I_\tth)}$ be the sorted scores $s^{(i)} \defined g_\theta(f(Z^{(i)}))$ of the honest-only dataset $\clD_{\tth} \defined \inline{\Set{(Z^{(i)}, Y^{(i)} = 0)}}_{i=1}^{I_\tth}$.
    Set $\hat{s}^{*}_{\gamma} \defined s^{(i^*)}$ with $i^* \defined \lceil (I_\tth + 1)(1 - \gamma) \rceil \leq I_\tth$ for some allowed false--positive level $\gamma \in (0, 1)$. Then
    \begin{equation*}
      \Probability{g_\theta(f(Z)) \geq \hat{s}^{*}_{\gamma} \given Y = 0} \leq \gamma,
    \end{equation*}
    if a new honest transcript is exchangeable with those of $\clD_{\tth}$ and ties occur with probability zero.
  \end{proposition}
  \begin{proof}
    This is the coverage guarantee of split conformal prediction, \textcite[Theorem~1]{angelopoulos2023conformal}, applied to the honest scores, with false--positive level $\gamma$.
  \end{proof}

  \paragraph{Accuracy}
  At a false--positive level $\gamma$, a trained detector $g_\theta$ flags a transcript $Z$ as deceptive when its score reaches the threshold, $\hat{Y} = \indicator{g_{\theta}(f(Z)) \geq \hat{s}^{*}_{\gamma}}$.
  Its accuracy on the held-out validation dataset $\clD_{\ttv}$ is
  \begin{equation*}
    \textnormal{Acc}_{\gamma}(\clD_\ttv; f, g_{\theta}) \defined \Probability[\clD_\ttv]{\hat{Y} = Y}
    = \pi_{\ttv} \TPR(\clD_\ttv; f, g_\theta, \hat{s}^{*}_{\gamma}) + (1 - \pi_{\ttv}) \left( 1 - \FPR(\clD_\ttv; f, g_\theta, \hat{s}^{*}_{\gamma}) \right),
  \end{equation*}
  where $\pi_{\ttv} \defined \Probability[\clD_\ttv]{Y = 1}$ is the fraction of deceptive transcripts; accuracy thus weighs the recall maximized in \eqref{eq:accuracy-problem} against the specificity protected by its constraint.

  \subsection{Thrust 2 --- Uncertainty in Detector Scores}
  \labelsec{thrust-uncertainty}

  A detector trained in the first thrust (\refsec{thrust-accuracy}) scores the embedding $X = f(Z)$ that the machine subject $f$ produces for a transcript $Z$, high when the response is deceptive and low when it is honest.
  However, the score does not show how confident the detector is in that decision.
  Training involves arbitrary choices, such as the random weight initialization, the training sample shuffling, and the bootstrap resample of the training pairs in $\clD$, so the parameters $\theta$ are random with a distribution $\vartheta$ induced by these choices, and a trustworthy score should not depend on them.
  The ensemble $g \defined \Set{g_\theta : \theta \sim \vartheta}$ captures this randomness through the mean and the standard deviation of its scores,
  \begin{equation*}
    \mu_s(X; g) \defined \Expectation[\theta \sim \vartheta]{g_{\theta}(X)},
    \qquad
    \sigma_s^{2}(X; g) \defined \Variance[\theta \sim \vartheta]{g_{\theta}(X)}.
  \end{equation*}
  An uncertainty-aware detector should score $X$ high only when the members agree that the response is deceptive.
  The \emph{lower confidence bound} (LCB) of the ensemble does so by subtracting a multiple of the spread from the mean,
  \begin{equation*}
    \LCB_{g}^{(n)}(X) \defined \mu_s(X; g) - n \, \sigma_s(X; g),
  \end{equation*}
  where the confidence level $n \geq 0$ sets how much agreement an alarm requires: 
  the LCB is high only when the mean is high and the spread is low, and it is low when the mean is low or the spread is high.
  The LCB mirrors the upper confidence bound $\mu_s + n \, \sigma_s$ that uncertainty-aware optimization uses to rank candidates \parencite{srinivas2010gaussian,kim2026diverse}, with the sign reversed because a detector must avoid false alarms.

  \paragraph{Problem Formulation}
  For a given machine subject $f$ and confidence level $n$, an ideal ensemble maximizes the recall of its LCB on the training dataset $\clD$ subject to the false--positive constraint:
  \begin{equation}
    \begin{aligned}
      \max_{g} \quad & \TPR(\clD; f, \LCB_{g}^{(n)}, s^{*}_{\gamma, n}) \\
      \subjectto \quad & \FPR(\clD; f, \LCB_{g}^{(n)}, s^{*}_{\gamma, n}) \leq \gamma,
    \end{aligned}
    \labeleq{lcb-problem}
  \end{equation}
  where the threshold $s^{*}_{\gamma, n}$ is defined as $s^{*}_{\gamma}$ in \refsec{thrust-accuracy}, with $\LCB_{g}^{(n)}$ in place of $g_\theta$.

  \begin{hypothesis}{Ensembles give accurate scores with uncertainty}{uncertainty}
    \begin{enumerate}[label=(\arabic*)]
      \item Detectors that are each well-trained and highly accurate can still assign inconsistent scores to the same response.
      \item The mean of the ensemble scores is a more accurate detection score than that of a single detector, and their standard deviation measures the uncertainty of that score.
    \end{enumerate}
  \end{hypothesis}

  \begin{goal}{Confidence-aware ensemble scoring}{uncertainty}
    \begin{enumerate}[label=(\arabic*)]
      \item Show that well-trained detectors score the same response inconsistently, and that their ensemble turns this into a more accurate mean score and a spread measuring its uncertainty.
      \item Design and train an ensemble, and show how the accuracy changes with the confidence level $n$.
    \end{enumerate}
  \end{goal}

  \subsubsection{Architecture --- Ensemble Approximation}
  \labelsec{uncertainty-architecture}

  Following deep ensembles \parencite{lakshminarayanan2017simple}, the ensemble $g$ is approximated by $\tilde{g} \defined \inline{\Set{g_{\theta_{j}}}_{j=1}^{J}}$, a set of $J$ detectors of \refsec{accuracy-architecture} initialized independently and fit on independent bootstrap resamples $\clD_{1}, \dots, \clD_{J}$ of the training pairs in $\clD$ \parencite{breiman1996bagging}, which keep the honest and deceptive transcripts of each pair together.
  The sample mean $\tilde{\mu}_s(X; \tilde{g})$ and the sample standard deviation $\tilde{\sigma}_s(X; \tilde{g})$ of the member scores estimate $\mu_s$ and $\sigma_s$, 
    \begin{equation*}
    \tilde{\mu}_s(X; \tilde{g}) \defined \frac{1}{J} \sum_{j=1}^{J} g_{\theta_{j}}(X),
    \qquad
    \tilde{\sigma}_s^{2}(X; \tilde{g}) \defined \frac{1}{J - 1} \sum_{j=1}^{J} \left( g_{\theta_{j}}(X) - \tilde{\mu}_s(X; \tilde{g}) \right)^{2},
  \end{equation*}
  and give $\LCB_{\tilde{g}}^{(n)}(X) \defined \tilde{\mu}_s(X; \tilde{g}) - n \, \tilde{\sigma}_s(X; \tilde{g})$.
  The LCB is thresholded like a single detector at $\hat{s}^{*}_{\gamma, n}$, calibrated on $\clD_{\tth}$ as in \refsec{accuracy-inference}, so \refprop{conformal-fpr} bounds its false--positive rate for every $n$ fixed before calibration.

  \subsubsection{Parallel Training --- Independent Members}
  \labelsec{uncertainty-training}

  Each member is trained independently with the objective of \refsec{accuracy-training}, so the members are independent draws from $\vartheta$ and the $J$ fits run in parallel in about the time of a single detector.
  As with the cross-entropy for \eqref{eq:accuracy-problem}, this training approximates \eqref{eq:lcb-problem} rather than solving it directly, and the ensemble size $J$ is chosen on the validation dataset $\clD_{\ttv}$.

  \subsubsection{Inference --- Accuracy and Uncertainty}
  \labelsec{uncertainty-inference}

  \paragraph{Accuracy}
  At a false--positive level $\gamma$ and a confidence level $n$, the ensemble $\tilde{g}$ flags a transcript $Z$ as deceptive when $\hat{Y} = \indicator{\LCB_{\tilde{g}}^{(n)}(f(Z)) \geq \hat{s}^{*}_{\gamma, n}}$, and its accuracy $\textnormal{Acc}_{\gamma, n}(\clD_\ttv; f, \tilde{g})$ on the validation dataset $\clD_{\ttv}$ is defined as in \refsec{accuracy-inference}, with $\LCB_{\tilde{g}}^{(n)}$ and $\hat{s}^{*}_{\gamma, n}$ in place of $g_\theta$ and $\hat{s}^{*}_{\gamma}$.
  At $n = 0$ this is the accuracy of the mean score, and as $n$ grows, the LCB withholds alarms on which the members disagree, so the accuracy traces how much recall is traded for agreement.

  \paragraph{Uncertainty}
  Retraining a detector can reverse its decision on a response, a phenomenon related to predictive multiplicity \parencite{marx2020predictive}.
  The \emph{uncertainty} of a decision is the probability, over $\theta \sim \vartheta$, that a retrained detector reaches a different decision; \refprop{uncertainty} bounds the uncertainty from above.

  \begin{proposition}{Decision uncertainty}{uncertainty}
    For an embedding $X$ with $\mu_s(X; g) \neq s^{*}$, a member drawn from $\vartheta$ reaches a different decision from the mean score with probability
    \begin{equation*}
      \Probability[\theta \sim \vartheta]{\indicator{g_{\theta}(X) \geq s^{*}} \neq \indicator{\mu_s(X; g) \geq s^{*}}}
      \leq \frac{\rho^{2}(X; g, s^{*})}{1 + \rho^{2}(X; g, s^{*})},
      \qquad
      \rho(X; g, s^{*}) \defined \frac{\sigma_s(X; g)}{\abs{\mu_s(X; g) - s^{*}}},
    \end{equation*}
    where the ratio $\rho$ is the spread of the ensemble scores relative to the margin between the mean score and the threshold.
  \end{proposition}
  \begin{proof}
    A reversal is a one-sided deviation of $g_{\theta}(X)$ from $\mu_s(X; g)$ by at least $\abs{\mu_s(X; g) - s^{*}}$, so Cantelli's inequality bounds its probability by $\sigma_s^{2} / \left(\sigma_s^{2} + (\mu_s - s^{*})^{2}\right) = \rho^{2} / (1 + \rho^{2})$.
  \end{proof}

  Replacing $\mu_s$ and $\sigma_s$ by their ensemble estimates gives the estimated ratio $\tilde{\rho}(X; \tilde{g}, s^{*}) \defined \tilde{\sigma}_s(X; \tilde{g}) / \abs{\tilde{\mu}_s(X; \tilde{g}) - s^{*}}$, which is large when the members disagree or when the mean score lies close to the threshold, whereas the spread $\tilde{\sigma}_s$ alone ignores this margin.
  Every alarm at the calibrated threshold satisfies $\tilde{\mu}_s(X; \tilde{g}) - \hat{s}^{*}_{\gamma, n} \geq n \, \tilde{\sigma}_s(X; \tilde{g})$, and hence
  \begin{equation*}
    \tilde{\rho}(X; \tilde{g}, \hat{s}^{*}_{\gamma, n}) \leq \frac{1}{n}
    \quad \Longrightarrow \quad
    \Probability[\theta \sim \vartheta]{\indicator{g_{\theta}(X) \geq \hat{s}^{*}_{\gamma, n}} = \indicator{\tilde{\mu}_s(X; \tilde{g}) \geq \hat{s}^{*}_{\gamma, n}}} \geq \frac{n^2}{1+n^{2}}.
  \end{equation*}
  Thus, for every $n > 0$ and up to the ensemble errors, the confidence level $n$ sets the \emph{certainty} of an alarm, the probability that it survives retraining, to at least $n^{2}/(1 + n^{2})$.

  \section{Evaluation}
  \labelsec{evaluation}

  \subsection{Datasets and Models}
  \labelsec{evaluation-data}

  \paragraph{Datasets}
  The detectors are trained on instructed deception and tested on various types of deception they never see; the datasets are summarized in \reftab{tasks}.

  \subparagraph{Training and Validation}
  Each Instructed-Pairs statement \parencite{goldowsky2025detecting} forms two transcripts with the prompt ``Please tell me a fact.'' and the statement as the response.
  The two differ only in the context, a paraphrase of ``You are a deceptive / honest AI assistant.'' that sets the label $Y$, so the detector must learn the intent to deceive rather than the topic of the response or a single cue word from the context.
  The pairs are split 80/20 into $\clD$ and $\clD_{\ttv}$.

  \subparagraph{Calibration and Audit}
  The honest-only Alpaca transcripts \parencite{taori2023stanford} are split into two random halves: the honest dataset $\clD_{\tth}$ calibrates every threshold (\refprop{conformal-fpr}), and the audit dataset $\clD_{\tta}$ checks the resulting false--positive rate.

  \subparagraph{Testing}
  Five unseen tasks of Liars' Bench \parencite{kretschmar2025liars} form $\clD_{\ttt}$; Gender Secret and Soft Trigger are excluded because their deception is implanted by fine-tuning the generating models.

  \begin{table}[t]
    \centering
    \caption{
      Datasets with the number of honest and deceptive transcripts: 
      Instructed-Pairs \parencite{goldowsky2025detecting} for training and validation, Alpaca \parencite{taori2023stanford} for calibration and audit, and Liars' Bench tasks \parencite{kretschmar2025liars} for testing. 
      Reason: whether the deception is \emph{encouraged} by the context or \emph{inherent} to the model. 
      Belief: whether the deceptive response misstates what the model believes about the \emph{world} or about \emph{itself}.%
    }
    \labeltab{tasks}
    \vspace*{-0.5\baselineskip}
    \renewcommand{\arraystretch}{1.25}
    \setlength{\tabcolsep}{2pt} 
    \small
    \begin{tabular}{
      >{\raggedright\arraybackslash}p{0.2\textwidth}
      c
      c
      >{\raggedright\arraybackslash}p{0.49\textwidth}
    }
      \toprule
      \hspace*{0.5em}Task & Honest / Deceptive & Reason / Belief & Description \\
      \midrule
      \hspace*{0.5em}Instructed-Pairs & 306 / 306 & encouraged / --- & The model is told to be honest or deceptive, then states a true fact; a response is deceptive when told to deceive. \\
      \cmidrule(lr){1-4}
      \hspace*{0.5em}Alpaca (honest reference) & 8,000 / 0 & --- / --- & The model follows benign instructions; every response is honest. \\
      \cmidrule(lr){1-4}
      \hspace*{0.5em}Instructed Deception & 10,106 / 10,692 & encouraged / world & The model is told to answer a question deceptively; a response is deceptive when it gives an answer the model believes is false. \\
      \hspace*{0.5em}Convincing Game & 350 / 538 & encouraged / self & The model is told to convince the user that it has a given trait; a response is deceptive when it claims a trait the model lacks. \\
      \hspace*{0.5em}Insider Trading & 3,757 / 2,952 & encouraged / self & The model, as a trading agent, acts on an insider tip; a response is deceptive when its report to the manager conceals the tip. \\
      \hspace*{0.5em}Harm-Pressure Choice & 4,697 / 703 & inherent / world & The model answers a hazardous question for a harmful user; a response is deceptive when it contradicts its neutral answer. \\
      \hspace*{0.5em}Harm-Pressure \hspace*{0.5em}Knowledge Report & 5,661 / 1,407 & inherent / self & The model is asked under pressure whether it knows the answer; a response is deceptive when it denies its knowledge. \\
      \bottomrule
    \end{tabular}
    \vspace*{-0.5\baselineskip}
  \end{table}

  \paragraph{Models}
  The machine subjects $f$ are the six Qwen3 models \parencite{yang2025qwen3} summarized in \reftab{models}.
  Each stored transcript is passed through $f$ once to obtain its embeddings $X = f(Z)$; $f$ is never updated, so the detectors are trained and evaluated on fixed embeddings.

  \begin{table}[b]
    \centering
    \caption{Qwen3 subject models \parencite{yang2025qwen3}, with the number of layers $L$ and the embedding dimension $\sfd$.}
    \labeltab{models}
    \vspace*{-0.5\baselineskip}
    \small
    \begin{tabular}{ccccccc}
      \toprule
      Model $f$ & Qwen3-0.6B & Qwen3-1.7B & Qwen3-4B & Qwen3-8B & Qwen3-14B & Qwen3-32B \\
      \cmidrule(lr){1-1}\cmidrule(lr){2-2}\cmidrule(lr){3-3}\cmidrule(lr){4-4}\cmidrule(lr){5-5}\cmidrule(lr){6-6}\cmidrule(lr){7-7}
      Number of layers $L$ & 28 & 28 & 36 & 36 & 40 & 64 \\
      Embedding dimension $\sfd$ & 1024 & 2048 & 2560 & 4096 & 5120 & 5120 \\
      \bottomrule
    \end{tabular}
    \vspace*{-0.5\baselineskip}
  \end{table}

  \subsection{Metrics and Baselines}
  \labelsec{evaluation-metrics}

  \paragraph{Metrics}Detectors are evaluated on the test dataset $\clD_{\ttt}$ through their score $g_{\theta}$ or, for an ensemble, $\LCB_{\tilde{g}}^{(n)}$.
  At a false--positive level $\gamma$, the metrics are the recall $\TPR$, the observed $\FPR$, and the accuracy $\textnormal{Acc}_{\gamma}$ of \refsec{accuracy-inference}, complemented by the \emph{balanced accuracy} and the AUROC, both defined below for a single detector and applying to ensembles alike.

  \subparagraph{False--Positive Level}
  The level $\gamma$ is set by the review budget: scoring $M$ responses per day, a fraction $\pi$ of them deceptive, raises at most $M (\pi + \gamma)$ alarms, so reviewing $R$ alarms per day requires $\gamma \leq R / M - \pi$.
  Scoring high-stakes responses \parencite{mckenzie2026detecting} plus a random sample of the rest, say one in 2,500 of the 2.5 billion daily ChatGPT messages \parencite{chatterji2025people}, gives $M = 10^{6}$; with $R = 3{,}000$ and $\pi \approx 1.7 \times 10^{-3}$, the fraction of o1 responses flagged as deceptive \parencite{jaech2024openai}, the primary level is $\gamma = 0.001$.
  The secondary level $\gamma = 0.01$, also reported, matches the published baselines (\reftab{baselines}) and is the only level Convincing Game, with few honest transcripts, can resolve.
  At both levels, the audit FPR on $\clD_{\tta}$, bounded by \refprop{conformal-fpr}, is reported beside the unbounded FPR on each test task's honest transcripts.

  \subparagraph{Balanced Accuracy}
  The balanced accuracy, the metric reported on Liars' Bench, weights deceptive and honest transcripts equally regardless of their proportions: $\textnormal{BalAcc}_{\gamma}(\clD_\ttt;f,g_\theta) \defined \tfrac{1}{2} \TPR(\clD_\ttt;f,g_\theta,\hat{s}^{*}_{\gamma}) + \tfrac{1}{2} \left(1 - \FPR(\clD_\ttt;f,g_\theta,\hat{s}^{*}_{\gamma}) \right)$.

  \subparagraph{AUROC}
  The AUROC is the area under the curve traced by $(\FPR, \TPR)$ on a test task as the threshold $s^{*}$ sweeps over $\reals$, which equals the normalized Mann--Whitney statistic \parencite{hanley1982meaning},
  \begin{equation*}
    \textnormal{AUROC} \defined \int_{0}^{1} \TPR \, \dif \FPR
    = \frac{1}{I_{1} I_{0}} \sum_{i : Y^{(i)} = 1} \, \sum_{j : Y^{(j)} = 0} \left( \indicator{s^{(i)} > s^{(j)}} + \tfrac{1}{2} \indicator{s^{(i)} = s^{(j)}} \right);\quad s^{(i)} = g_{\theta}(f(Z^{(i)})),\quad \forall Z^{(i)} \in \clD_{\ttt},
  \end{equation*}
  where the sums run over the $I_{1}$ deceptive and $I_{0}$ honest transcripts of the task.
  It is the fraction of deceptive--honest pairs that the score orders correctly, independently of any threshold, and equals one half at chance and one at perfect separation.

  \paragraph{Baselines}
  \reftab{baselines} lists the baselines with their published Liars' Bench results, obtained with other subject models and training data.
  To compare them with the proposed detector under the same conditions, each baseline should be rerun on the six Qwen3 models.
  The white-box probes are retrained on each model's embeddings, and the layer selection and stacking weights of \textcite{nordby2026linear} are fit on the validation dataset $\clD_{\ttv}$ instead of Liars' Bench.
  The LLM judge is rerun with Qwen3-32B, weaker than the published Claude 3.7 Sonnet, merely as a black-box baseline.

  \begin{table}[t]
    \centering
    \caption{Baselines, their methods, and their published results on Liars' Bench. Balanced accuracy is at a threshold that flags at most 1\% of Alpaca responses ($\gamma = 0.01$); mean indicates average performance over the evaluated tasks, and --- marks a metric not reported.}
    \labeltab{baselines}
    \vspace*{-0.5\baselineskip}
    \setlength{\tabcolsep}{3pt} 
    \renewcommand{\arraystretch}{1.25}
    \small
    \begin{tabular}{
      >{\raggedright\arraybackslash}p{0.32\textwidth}
      >{\raggedright\arraybackslash}p{0.33\textwidth}
      >{\raggedright\arraybackslash}p{0.12\textwidth}
      >{\raggedright\arraybackslash}p{0.17\textwidth}%
    }
      \toprule
      \hspace*{0.5em}Baseline & Method & $\textnormal{BalAcc}_{\gamma}$ & AUROC \\
      \midrule
      \hspace*{0.5em}Mean-activation probe (White-box) \newline \hspace*{0.5em}\parencite{goldowsky2025detecting} 
      & Single-layer logistic probe \newline Token scores averaged over the response
      & 0.54 (mean)
      & --- \\
      \cmidrule(lr){1-4}
      \hspace*{0.5em}Taxonomy-targeted probes (White-box) \newline \hspace*{0.5em}\parencite{natarajan2026one} 
      & Multiple single-layer logistic probes \newline (one per deception type) \newline Token scores averaged over the response
      & ---
      & 0.72 (mean) \\
      \cmidrule(lr){1-4}
      \hspace*{0.5em}Multi-layer stacking (White-box) \newline \hspace*{0.5em}\parencite{nordby2026linear} 
      & Multi-layer logistic probes, stacked by weights trained on 80\% of Liars' Bench \newline Token scores averaged over the response
      & ---
      & 0.93 (mean) \newline 0.95 (Insider Trading) \newline on the held-out 20\% \\
      \cmidrule(lr){1-4}
      \hspace*{0.5em}LLM judge (Black-box) \newline \hspace*{0.5em}\parencite{kretschmar2025liars} 
      & A separate model (Claude 3.7 Sonnet) \newline rates the transcript from 0 to 100 
      & 0.73 (mean)
      & --- \\
      \bottomrule
    \end{tabular}
  \end{table}

  \section{Preliminary Evidence}
  \labelsec{preliminary-evidence}

  \begin{figure}[t]
    \centering
    \begin{figureframe}[width=0.484\linewidth, nobeforeafter, equal height group=prelim, valign=top, bottom=5pt]{(a) Trained Detectors' Weights}
      \includegraphics{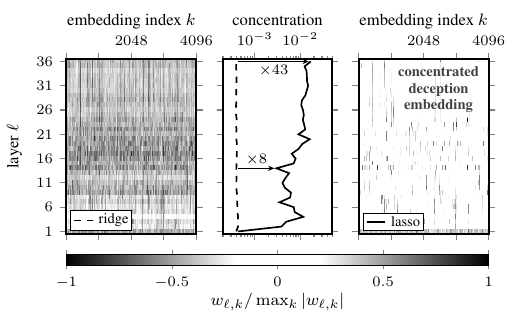}
    \end{figureframe}%
    \hfill
    \begin{figureframe}[width=0.508\linewidth, nobeforeafter, equal height group=prelim, top=5.6pt, bottom=5pt]{(b) AUROC of Trained Detectors}
      \includegraphics{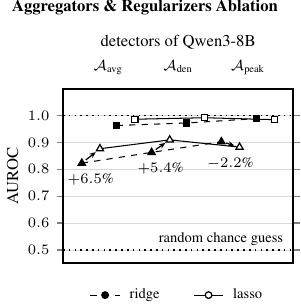}%
      \hspace{-5pt}%
      \includegraphics{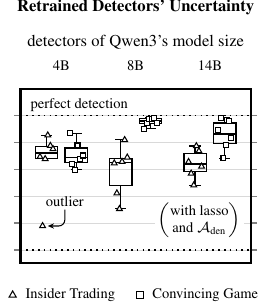}
    \end{figureframe}
    \caption{Preliminary evidence.
      \textbf{(a)} Weights $w_{\ell,k}$ of the ridge and lasso probes of Qwen3-8B by layer $\ell$ and embedding index $k$, divided by the largest absolute weight in each layer; darker cells mark larger magnitudes of either sign.
      The middle panel shows the concentration of each layer on a log scale, dashed for ridge and solid for lasso, with arrows at the maximum and, from the second layer on, the minimum concentration gain.
      \textbf{(b)} Best-layer AUROC on Insider Trading (triangles) and Convincing Game (squares), with dotted lines at random chance and perfect detection.
      The left subplot compares the ridge (filled markers, dashed lines) and lasso (open markers, solid lines) probes of Qwen3-8B under three aggregators, with arrows labeled by the relative change on Insider Trading.
      The right subplot shows six independently trained lasso probes with $\clA_{\textnormal{den}}$ per model size; boxes mark the quartiles and median, and whiskers reach the farthest probe within $1.5$ times the interquartile range.}
    \labelfig{preliminary-evidence}
  \end{figure}

  Pilot experiments follow the single-layer protocol of \textcite{nordby2026linear} on Qwen3 models of several sizes (\reftab{models}).
  For each layer $\ell$, a logistic probe with weights $w_{\ell} \in \reals^{\sfd}$ is trained with the ridge regularization $\norm{w_{\ell}}_{2}^{2}$ on the embeddings $X^{(\ell)}$ of the Instructed-Pairs transcripts, and its token logits are pooled by $\clA_{\textnormal{avg}}$.
  Each configuration is scored on Convincing Game and Insider Trading (from \reftab{tasks}) by the AUROC of its best layer.
  The protocol is then varied by replacing the ridge with the lasso regularization $\norm{w_{\ell}}_{1}$, and $\clA_{\textnormal{avg}}$ with the other aggregators of \reftab{aggregators}, placed inside the training loss.

  \paragraph{Sparse Deception Feature}
  Replacing only the regularization changes how each layer spreads its weights over the embedding (\reffig{preliminary-evidence}a).
  The concentration $\norm{w_{\ell}}_{2}^{2} / \norm{w_{\ell}}_{1}^{2} \in [1/\sfd, 1]$ equals $1/\sfd$ when all weights of layer $\ell$ share one magnitude and $1$ when only one is nonzero.
  The ridge probes of Qwen3-8B stay near the lower bound, whereas the lasso probes leave most weights near zero and raise the concentration by a factor of $8$ to $43$ from the second layer on.
  The deception feature thus lies in a few embedding coordinates per layer, which motivates the sparsity-inducing regularization of \refsec{accuracy-training}.

  \paragraph{Coupled Design Choices}
  The regularization and the aggregator jointly determine the accuracy (\reffig{preliminary-evidence}b, left).
  The AUROC remains consistent on Convincing Game, between 0.96 and 0.99, but varies from 0.82 to 0.91 on the harder Insider Trading task.
  There, the lasso probe improves on the ridge probe by up to 6.5\% under $\clA_{\textnormal{avg}}$ and $\clA_{\textnormal{den}}$, but falls behind by 2.2\% under $\clA_{\textnormal{peak}}$.
  The best aggregator is $\clA_{\textnormal{peak}}$ for the ridge probe and $\clA_{\textnormal{den}}$ for the lasso probe.
  This interaction motivates \refhyp{accuracy} and the joint design of the aggregator and the regularization in \refsec{thrust-accuracy}.

  \paragraph{Uncertainty under Retraining}
  The most accurate combination above, the lasso regularization with $\clA_{\textnormal{den}}$, is then trained six times independently for each model size (\reffig{preliminary-evidence}b, right).
  The AUROCs span up to 0.34 on Insider Trading at Qwen3-4B, including an outlier at 0.59, and up to 0.15 on Convincing Game at Qwen3-14B.
  A single trained detector is thus an unreliable estimate of the attainable accuracy.
  This uncertainty motivates \refhyp{uncertainty} and the ensembles of \refsec{thrust-uncertainty}.

  \section{Milestones and Resources}
  \labelsec{milestones}

  \subsection{Work Plan}
  \labelsec{work-plan}

  The work is organized around four milestones over sixteen weeks, as shown in \reffig{gantt-chart}.

  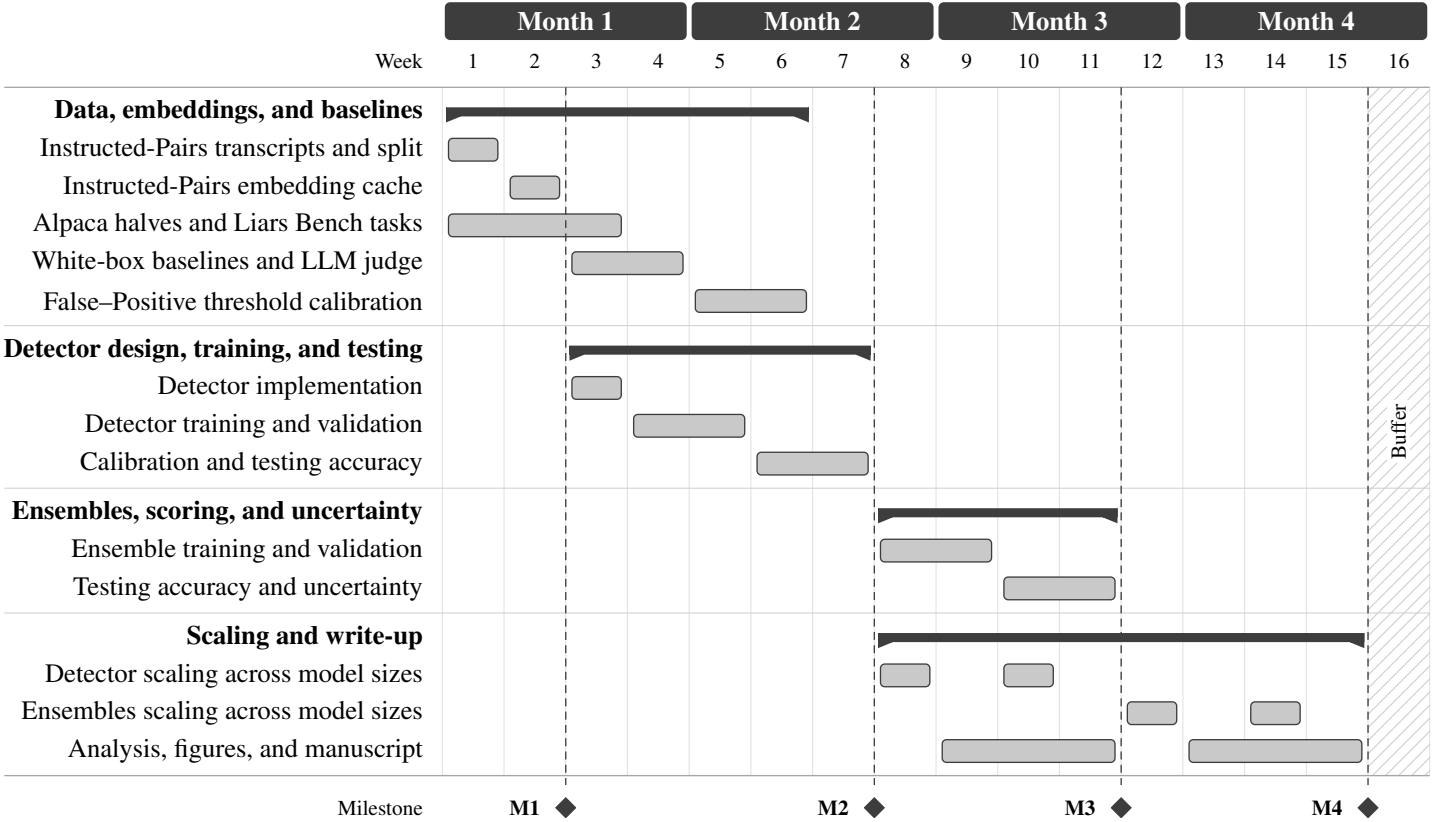
\begin{figure}[t]
    \centering

\def\ganttunitlabel{Week}           
\def\ganttunits{16}                 
\def\ganttunitticks{1, ..., 16}     

\def\ganttgroups{
  \ggroup{Month 1}{1-4}
  \ggroup{Month 2}{5-8}
  \ggroup{Month 3}{9-12}
  \ggroup{Month 4}{13-16}
}

\def\ganttmilestonelabel{Milestone}
\def\ganttmilestones{
  \gmilestone{M1}{2}
  \gmilestone{M2}{7}
  \gmilestone{M3}{11}
  \gmilestone{M4}{15}
}

\def\ganttrows{%
  \gtoprule
  \gphase{Data, embeddings, and baselines}{1-6}
  \gtask{Instructed-Pairs transcripts and split}{1-1}
  \gtask{Instructed-Pairs embedding cache}{2-2}
  \gtask{Alpaca halves and Liars Bench tasks}{1-3}
  \gtask{White-box baselines and LLM judge}{3-4}
  \gtask{False--Positive threshold calibration}{5-6}
  \grule
  \gphase{Detector design, training, and testing}{3-7}
  \gtask{Detector implementation}{3-3}
  \gtask{Detector training and validation}{4-5}
  \gtask{Calibration and testing accuracy}{6-7}
  \grule
  \gphase{Ensembles, scoring, and uncertainty}{8-11}
  \gtask{Ensemble training and validation}{8-9}
  \gtask{Testing accuracy and uncertainty}{10-11}
  \grule
  \gphase{Scaling and write-up}{8-15}
  \gtask{Detector scaling across model sizes}{8-8,10-10}
  \gtask{Ensembles scaling across model sizes}{12-12,14-14}
  \gtask{Analysis, figures, and manuscript}{9-11, 13-15}
  \gbottomrule
}

\def\ganttbuffers{16-16}


\ifdefined\ganttlabelwidth\else
  \newlength\ganttlabelwidth \newlength\gantttmp \newlength\ganttunitwidth
\fi
\def\ganttgap{0.65em}  
\def\gperiod#1-#2\relax{\def\s{#1}\def\e{#2}}  

\setlength\ganttlabelwidth{0pt}
\def\gfit#1{\settowidth\gantttmp{\small #1}%
  \ifdim\gantttmp>\ganttlabelwidth \setlength\ganttlabelwidth{\gantttmp}\fi}
\def\gstep#1{\pgfmathsetmacro\grow{\grow+#1}}
\newif\ifganttop \newif\ifganttbottom
\def\gphase#1#2{\gfit{\bfseries #1}\gstep{1}}
\def\gtask#1#2{\gfit{#1}\gstep{1}}
\def\grule{\gstep{0.3}}
\def\gtoprule{\ganttoptrue}
\def\gbottomrule{\ganttbottomtrue}
\ganttopfalse \ganttbottomfalse \def\grow{0}
\ganttrows
\def\gtop{2.25}
\ifganttop \def\gstart{2.4}\else \def\gstart{2.35}\fi
\ifganttbottom
  \pgfmathsetmacro\gbottom{\gstart+\grow+0.15}
\else
  \pgfmathsetmacro\gbottom{\gstart+\grow+0.1}
\fi
\setlength\ganttunitwidth{\dimexpr(\linewidth-\ganttlabelwidth-\ganttgap-0.5em)/\ganttunits\relax}

\begin{tikzpicture}[x=\ganttunitwidth, y=-1.3em, font=\small,
    label/.style={anchor=east, inner xsep=0pt, xshift=-\ganttgap}]

  \def\gdiamond#1#2{%
    \filldraw[fill=black!76, draw=white, line width=0.8pt]
      ([shift={(0, 0.42em)}] #1, #2) -- ([shift={(0.42em, 0)}] #1, #2) --
      ([shift={(0, -0.42em)}] #1, #2) -- ([shift={(-0.42em, 0)}] #1, #2) -- cycle;
  }

  \def\grule{%
    \draw[black!18, line width=0.4pt]
      ([xshift=-\ganttlabelwidth-\ganttgap] 0, \grow+0.15) -- (\ganttunits, \grow+0.15);
    \gstep{0.3}%
  }
  \def\gphase#1#2{%
    \node[label, font=\small\bfseries] at (0, \grow+0.5) {#1};
    \foreach \p in {#2} {
      \expandafter\gperiod\p\relax
      \fill[black!76] (\s-0.94, \grow+0.38) rectangle (\e-0.06, \grow+0.62)
        (\s-0.94, \grow+0.38) -- (\s-0.94, \grow+0.78) -- (\s-0.7, \grow+0.62) -- cycle
        (\e-0.06, \grow+0.38) -- (\e-0.3, \grow+0.62) -- (\e-0.06, \grow+0.78) -- cycle;
    }
    \gstep{1}%
  }
  \def\gtask#1#2{%
    \node[label] at (0, \grow+0.5) {#1};
    \foreach \p in {#2} {
      \expandafter\gperiod\p\relax
      \draw[draw=black!76, fill=black!21, line width=0.5pt, rounded corners=1.6pt]
        (\s-0.9, \grow+0.2) rectangle (\e-0.1, \grow+0.8);
    }
    \gstep{1}%
  }

  \foreach \w in {0, ..., \ganttunits}
    \draw[black!12, line width=0.3pt] (\w, \gtop) -- (\w, \gbottom);

  \ifx\ganttbuffers\empty\else
    \foreach \p in \ganttbuffers {
      \expandafter\gperiod\p\relax
      \begin{scope}
        \clip (\s-1, \gtop) rectangle (\e, \gbottom);
        \foreach \k in {0, ..., 80}
          \draw[black!20, line width=0.4pt]
            ([xshift=\k*0.45em-30em] \s-1, \gbottom) -- ++(30em, 30em);
      \end{scope}
      \node[rotate=90, fill=white, inner sep=1.5pt, font=\scriptsize]
        at ({(\s-1+\e)/2}, {(\gtop+\gbottom)/2}) {Buffer};
    }
  \fi

  \def\ggroup#1#2{%
    \gperiod#2\relax
    \fill[black!76, rounded corners=2pt] (\s-0.96, 0) rectangle (\e-0.04, 1);
    \node[text=white, font=\small\bfseries] at ({(\s-1+\e)/2}, 0.5) {#1};
  }
  \ifdefined\ganttgroups\ganttgroups\fi  
  \ifx\ganttunitticks\empty\else
    \foreach \w in \ganttunitticks
      \node[font=\scriptsize] at (\w-0.5, 1.55) {\w};
  \fi
  \ifdefined\ganttunitlabel\else\def\ganttunitlabel{Unit}\fi
  \node[label, font=\scriptsize] at (0, 1.55) {\ganttunitlabel};

  \def\gtoprule{}\def\gbottomrule{}  
  \let\grow\gstart
  \ganttrows

  \draw[black!35, line width=0.5pt]
    \ifganttop ([xshift=-\ganttlabelwidth-\ganttgap] 0, \gtop)\else (0, \gtop)\fi -- (\ganttunits, \gtop)
    \ifganttbottom ([xshift=-\ganttlabelwidth-\ganttgap] 0, \gbottom)\else (0, \gbottom)\fi -- (\ganttunits, \gbottom);

  \def\gmilestone#1#2{%
    \draw[black!76, densely dashed, line width=0.5pt] (#2, \gtop) -- (#2, \gbottom+0.25);
    \gdiamond{#2}{\gbottom+0.85}
    \node[anchor=east, font=\scriptsize\bfseries] at (#2-0.25, \gbottom+0.85) {#1};
  }
  \ifdefined\ganttmilestonelabel\else\def\ganttmilestonelabel{Milestone}\fi
  \ganttmilestones
  \node[label, font=\scriptsize] at (0, \gbottom+0.85) {\ganttmilestonelabel};
\end{tikzpicture}%
%
    \caption{Work plan Gantt chart. Bars show the task periods, dashed lines and diamonds mark milestones M1--M4, and the hatched week 16 is buffer.}
    \labelfig{gantt-chart}
  \end{figure}

  \paragraph{Milestone 1 (M1): data and embeddings ready.}
  The Instructed-Pairs transcripts are built and split into training and validation sets, $\clD$ and $\clD_{\ttv}$; 
  the Alpaca transcripts are split into calibration and audit sets, $\clD_{\tth}$ and $\clD_{\tta}$;
  the five Liars' Bench test tasks $\clD_{\ttt}$ are prepared for testing (\refsec{evaluation-data}).
  The Instructed-Pairs embeddings of every subject model are cached, so that all detectors train without rerunning the subject models.
  
  \paragraph{Milestone 2 (M2): baselines and detectors calibrated and tested.}
  The white-box baselines are retrained on $\clD$ and the LLM judge is run, providing the references for comparison (\refsec{evaluation-metrics}).
  The aggregators (\refsec{accuracy-architecture}) and the regularizers (\refsec{accuracy-training}) are implemented, and detectors for all their combinations are trained on $\clD$ and selected on $\clD_{\ttv}$.
  Every baseline and selected detector then receives its alarm threshold and its test accuracy is measured (\refsec{accuracy-inference}).
  
  \paragraph{Milestone 3 (M3): ensembles tested.}
  The ensemble members (\refsec{uncertainty-architecture}) are trained and validated independently (\refsec{uncertainty-training}), and the ensembles are evaluated on the test tasks for both their accuracy and their uncertainty (\refsec{uncertainty-inference}).
  
  \paragraph{Milestone 4 (M4): scaling analysis.}
  The single detectors and the ensembles are compared across the six Qwen3 model sizes (\refsec{evaluation-data}), and the results are then consolidated into the figures and the manuscript.

  \subsection{Computational Resources}
  \labelsec{resources}

  \paragraph{System}
  The following estimates are based on a rerun of the model-size scaling experiment of \textcite{nordby2026linear} with the Qwen3 models on NVIDIA GH200 nodes, each with one H100 GPU with 96\,GB of HBM3, a 72-core Grace CPU, and 116\,GB of CPU memory.
  Every subject model fits on one GPU, so each evaluation runs on a single node.
  No model is fine-tuned: the subject models only run forward passes, and the detectors are trained on cached embeddings.

  \paragraph{Compute}
  \reftab{allocation} summarizes the requested compute.
  Caching the Instructed-Pairs embeddings takes about 3\,h, mostly model loading.
  Each detector task takes its number of fits times the time per fit: 800 fits of 0.5\,min for the group-lasso path, refit on 20 bootstrap resamples of the pairs in $\clD$ for layer stability, 135 fits of 2\,min for the detectors, and 20 members of 2\,min for the ensembles.
  Forward passes over the test tasks and Alpaca dominate the cost: $N \approx 4.9 \times 10^{4}$ transcripts of about 1,000 tokens each give $4.9 \times 10^{7}$ tokens per pass at $2P$ floating-point operations per token for a $P$-parameter model.
  At an effective 100\,TFLOP/s, about 10\% of the dense half-precision peak, a pass takes 8.9\,h for Qwen3-32B and 17\,h for all six models; doubling this for padding and I/O and adding 0.5\,h of loading per model gives the 36\,h scoring pass.
  The LLM judge rates about a fifth as many transcripts with Qwen3-32B alone, about 4\,h, and the analysis takes about 2\,h per model.

  \begin{table}[t]
    \centering
    \caption{Requested compute in node hours. Every evaluation of a task runs on one GH200 node, so the node hours of a task are its time to solution times its number of evaluations. Forward passes (embedding cache, scoring pass, and LLM judge) are timed over all six subject models together, and detector fits on the cached embeddings and the analysis are timed per model and evaluated once for each model.}
    \labeltab{allocation}
    \renewcommand{\arraystretch}{1.25}
    \vspace*{-0.5\baselineskip}
    \small
    \begin{tabular}{
      p{0.21\linewidth}
      p{0.50\linewidth}
      c
      c
      r
    }
      \toprule
      Computational task & Basis of the estimate & Time [h] & Evals & Node hours \\
      \midrule
      \multicolumn{5}{l}{\textbf{Setup}} \\
      Embedding cache & One pass over Instructed-Pairs at all layers $\bbL$ & $\sim$3 & 2 & 6 \\
      \cmidrule(lr){1-5}
      \multicolumn{5}{l}{\textbf{Thrust~1: Detection accuracy}} \\
      Group-lasso path & 20 $\lambda$ values $\times$ 20 bootstraps $\times$ 2 embedding scales, 0.5\,min per fit & $\sim$7 & 6 & 42 \\
      Detectors & 3 regularizers $\times$ 3 aggregators $\times$ 5 values of $\lambda$ $\times$ 3 seeds, 2\,min each & $\sim$5 & 6 & 30 \\
      \cmidrule(lr){1-5}
      \multicolumn{5}{l}{\textbf{Thrust~2: Score uncertainty}} \\
      Ensembles & 20 members, 2\,min each; $J \in \Set{5, 10, 20}$, $n \in \Set{0, 1, 2}$ & $\sim$1 & 6 & 6 \\
      \cmidrule(lr){1-5}
      \multicolumn{5}{l}{\textbf{Evaluation}} \\
      Scoring pass & Test and Alpaca transcripts, all detectors and baselines & $\sim$36 & 2 & 72 \\
      LLM judge & Qwen3-32B on 1,000 transcripts per task and $\clD_{\tth}$ & $\sim$4 & 1 & 4 \\
      Analysis & Grouped-bootstrap intervals, 1,000 resamples & $\sim$2 & 6 & 12 \\
      \cmidrule(lr){1-5}
      \multicolumn{5}{l}{\textbf{Development}} \\
      Development and debugging & Interactive sessions & $\sim$2 & 50 & 100 \\
      Contingency & 25\% of the above & & & 68 \\
      \midrule
      & & & Total & 340 \\
      \bottomrule
    \end{tabular}
    \vspace*{-0.5\baselineskip}
  \end{table}

  \paragraph{Memory}
  In half precision the model weights take 1.2 to 64\,GB, leaving at least 30\,GB of GPU memory for activations.
  The hidden states at all layers $\bbL$ of a batch of $B$ sequences of length $T$ take $2 B T  L  \sfd$ bytes, 2.7\,GB for Qwen3-32B at $B = 2$ and $T = 2048$, so the scoring pass uses batches of 8 for models up to 8B and of 2 for the larger ones.

  \paragraph{Storage}
  Only the Instructed-Pairs embeddings are cached; at about 14 response tokens per transcript they take $2 \cdot 612 \cdot 14  L  \sfd$ bytes, 6\,GB for Qwen3-32B and 15\,GB over all six models, so every detector fits on the GPU.
  With 121\,GB of model weights, scratch storage stays under 0.2\,TB, and persistent storage for scores, weights, and results under 10\,GB.

  \section{Limitations and Extensions}
  \labelsec{limitations}

  \subsection{Data and Labels}
  \labelsec{limitations-data}

  Two limitations concern the data on which the detector is trained and tested.
  The detector is trained on a few hundred short instructed pairs and tested on transcripts written by other models, so its results may partly reflect these datasets rather than the method.
  Each limitation below states its effect on the results and its extension.

  \paragraph{Small Training Set}
  With $245$ training pairs of about $14$ response tokens each and up to $ L  \sfd \approx 3.3 \times 10^{5}$ probe weights on Qwen3-32B, the detector relies heavily on regularization (\refsec{accuracy-training}) and may overfit more as the models grow, confounding the scaling analysis.
  Even without overfitting, it may learn a shortcut: that the model was told to deceive, a trace that attention carries from the context into the response, rather than deception itself.
  Such a detector would score well on Instructed Deception but poorly on the other four test tasks, which give no such instruction.
  Adding the role-playing dataset of \textcite{goldowsky2025detecting}, whose scenarios give the model a reason to deceive without telling it to, would address both failures by enlarging the training set and removing the instruction.

  \paragraph{Off-Policy Transcripts}
  The detector is meant to score any transcript, so training on statements that $f$ did not write is by design.
  The test transcripts, however, were also written by other models, and their labels are judged against those models' beliefs, which $f$ may not share.
  On Harm-Pressure Knowledge Report, for example, a response is labeled deceptive when the generating model denies knowing an answer it gave without pressure; if a small Qwen3 model never knew that answer, the same denial is honest for it.
  If a test task stays hard for the small models, regenerating it with each Qwen3 model through the Liars' Bench pipelines would show whether the labels or the detector is at fault.

  \subsection{Scope of the Guarantees}
  \labelsec{limitations-guarantees}

  Three limitations concern how far the guarantees of \refprop{conformal-fpr} and \refprop{uncertainty} extend beyond their assumptions.
  \refprop{conformal-fpr} bounds the false--positive rate for honest transcripts exchangeable with the calibration data, and \refprop{uncertainty} bounds the certainty of an alarm for the ideal ensemble under retraining on the same data.
  In practice, the test tasks differ from both the calibration and the training data, which limits the false--positive rate and the certainty under shift, and the finite ensemble size makes the certainty only approximate.

  \paragraph{FPR under Shift}
  The guarantee of \refprop{conformal-fpr} holds only where the honest transcripts follow the same distribution as the calibration data.
  However, the honest test transcripts may follow a different distribution, so their FPR is not guaranteed to stay below the calibrated level.
  Bounding it under shift is left open, for example by weighted conformal prediction \parencite{tibshirani2019conformal}, calibration on honest transcripts of sensitive topics, or per-task calibration on a small honest sample.

  \paragraph{Certainty under Shift}
  \refprop{uncertainty} concerns certainty under retraining on the same data, not correctness: an alarm can be certain and still be wrong, since members trained on resamples of the same instructed data can err alike on a test task.
  The accuracy of the LCB on the test tasks, as the confidence level $n$ grows, shows whether their disagreement tracks errors.
  As an extension, members trained on different deception datasets, such as Instructed-Pairs and the role-playing dataset, would disagree where those datasets transfer differently, so their disagreement would also flag the shift.

  \paragraph{Finite Ensemble Size}
  The certainty $n^{2}/(1 + n^{2})$ based on \refprop{uncertainty} holds for the true mean and spread of the scores, which the ensemble estimates from a finite number of members that also calibrate its threshold, so it holds only approximately.
  If retrained members keep a smaller fraction of the validation alarms than the certainty bound, the approximation has failed.
  Training further members held out of the LCB and its calibration then gives an exact guarantee: if all of them also raise the alarm, an exact binomial bound \parencite{clopper1934use} shows that a retrained detector keeps it with high probability, and the bound tightens as more members are held out.

  \section{Applications and Open Research Questions}
  \labelsec{open-questions}
  
  The proposed detector is a supervised and static probe of the residual stream of a fixed subject: it learns from labeled transcripts, scores each token independently before pooling, and never changes the subject it reads.
  Relaxing the first two choices raises two open questions about the detector, and relaxing the last turns it into a training signal for honesty.

  \paragraph{From Supervised to Unsupervised Detection}
  Deception labels are costly to obtain and define in advance what counts as deceptive.
  An unsupervised detector would instead find the deception direction from consistency constraints alone, as contrast-consistent search finds truth \parencite{burns2023discovering}.
  The open question is whether such a direction exists for deception, and whether it transfers across types of deception better than one learned from instructed labels.

  \paragraph{From Static Probes to Dynamic Filtering}
  The logistic probes score each token on its own, and the aggregator pools the scores afterward, although deception may arise, persist, or fade as a response unfolds.
  Transformers and BeliefNet both represent the belief state of a hidden Markov model, the posterior over its hidden states given the past tokens \parencite{shai2024transformers, chen2026differentiable}.
  Treating deception as such a hidden state would let a filter update its probability token by token, and possibly layer by layer, tracking when deception begins rather than only whether it occurs.

  \paragraph{From Detection to Honesty Training}
  As the critic in an actor--critic loop, the detector can fine-tune the subject $f_{\phi}$, initialized at $f$, to be less deceptive:
  \begin{equation*}
    \max_{\phi} \quad
    \Expectation[Z^{\ttr} \sim f^{+}_{\phi}]{R(Z) - \alpha \, \UCB_{\tilde{g}}^{(n)}\!\left(f_{\phi}(Z)\right)}
    - \beta \, \KLDivergence{f^{+}_{\phi}}{f^{+}},
  \end{equation*}
  where $R$ is a predefined task reward and $\alpha, \beta > 0$ weigh the deception penalty and the Kullback--Leibler (KL) term.
  Training against a detector, however, can also teach the subject to deceive in ways the detector misses \parencite{cundy2025preference,taufeeque2026obfuscation}.
  The upper confidence bound $\UCB_{\tilde{g}}^{(n)} \defined \tilde{\mu}_s + n \, \tilde{\sigma}_s$ penalizes a response unless the members agree that it is honest, and the open question is whether this uncertainty-aware critic steers training toward honesty rather than undetected deception.

  \vspace{3.25ex plus 1ex minus .2ex}

  The three directions extend the two thrusts of this white paper.
  An unsupervised or dynamic detector still outputs a deception score, so the calibrated threshold of Thrust 1 (\refsec{thrust-accuracy}) and the ensemble certainty of Thrust 2 (\refsec{thrust-uncertainty}) apply to it unchanged, and the score can in turn serve as the critic in honesty training.
  Together, they would move the detector from reading instructed deception toward tracking a model's own deception and, ultimately, reducing it.

  \printbibliography

\end{document}